\documentclass[sigconf]{acmart}

\AtBeginDocument{%
  \providecommand\BibTeX{{%
    \normalfont B\kern-0.5em{\scshape i\kern-0.25em b}\kern-0.8em\TeX}}}

\usepackage{amsthm,amsmath,amsfonts,bm,bbm}
\usepackage{physics,graphicx}
\usepackage{xcolor}
\usepackage[shortlabels]{enumitem}
\usepackage[linesnumbered,ruled,vlined]{algorithm2e}
\usepackage{multirow}

\newtheorem{prop}{Proposition}

\newtheorem{lemma}{Lemma}
\newtheorem{definition}{Definition}
\usepackage{subcaption}

\newcommand{\vpara}[1]{\vspace{0.05in}\noindent\textbf{#1 }}
\newcommand*\phantomrel[1]{\mathrel{\phantom{#1}}}

\newcommand{\N}{\lfloor N \rfloor}

\newcommand{\mG}{\mathcal{G}}

\usepackage{accents}
\newcommand{\ubar}[1]{\underaccent{\bar}{#1}}

\copyrightyear{2022} 
\acmYear{2022} 
\setcopyright{acmlicensed}\acmConference[WSDM '22]{Proceedings of the Fifteenth ACM International Conference on Web Search and Data Mining}{February 21--25, 2022}{Tempe, AZ, USA}
\acmBooktitle{Proceedings of the Fifteenth ACM International Conference on Web Search and Data Mining (WSDM '22), February 21--25, 2022, Tempe, AZ, USA}
\acmPrice{15.00}
\acmDOI{10.1145/3488560.3498506}
\acmISBN{978-1-4503-9132-0/22/02}
\begin{document}
\fancyhead{}
\title{Fast Learning of MNL Model from General Partial Rankings with Application to Network Formation Modeling}

\author{Jiaqi Ma}
\authornote{Both authors contributed equally to this research.}
\email{jiaqima@umich.edu}
\affiliation{%
  \institution{School of Information\\ University of Michigan}
  \city{Ann Arbor}
  \state{Michigan}
  \country{USA}
}

\author{Xingjian Zhang}
\authornotemark[1]
\email{jimmyzxj@umich.edu}
\affiliation{%
  \institution{Department of EECS\\ University of Michigan}
  \city{Ann Arbor}
  \state{Michigan}
  \country{USA}
}

\author{Qiaozhu Mei}
\email{qmei@umich.edu}
\affiliation{%
  \institution{School of Information and Department of EECS\\ University of Michigan}
  \city{Ann Arbor}
  \state{Michigan}
  \country{USA}
}

% \renewcommand{\shortauthors}{Ma, Zhang, and Mei.}

\begin{abstract}
Multinomial Logit (MNL) is one of the most popular discrete choice models and has been widely used to model ranking data. However, there is a long-standing technical challenge of learning MNL from many real-world ranking data: exact calculation of the MNL likelihood of \emph{partial rankings} is generally intractable. In this work, we develop a scalable method for approximating the MNL likelihood of general partial rankings in polynomial time complexity. We also extend the proposed method to learn mixture of MNL. We demonstrate that the proposed methods are particularly helpful for applications to choice-based network formation modeling, where the formation of new edges in a network is viewed as individuals making choices of their friends over a candidate set. The problem of learning mixture of MNL models from partial rankings naturally arises in such applications. And the proposed methods can be used to learn MNL models from network data without the strong assumption that temporal orders of all the edge formation are available. We conduct experiments on both synthetic and real-world network data to demonstrate that the proposed methods achieve more accurate parameter estimation and better fitness of data compared to conventional methods\footnote{Code available at \url{https://github.com/xingjian-zhang/Fast-Partial-Ranking-MNL}.}.
\end{abstract}

\begin{CCSXML}
<ccs2012>
   <concept>
       <concept_id>10002951.10003317.10003338.10003343</concept_id>
       <concept_desc>Information systems~Learning to rank</concept_desc>
       <concept_significance>500</concept_significance>
       </concept>
   <concept>
       <concept_id>10002951.10003317.10003338.10003340</concept_id>
       <concept_desc>Information systems~Probabilistic retrieval models</concept_desc>
       <concept_significance>500</concept_significance>
       </concept>
   <concept>
       <concept_id>10002951.10003260.10003261</concept_id>
       <concept_desc>Information systems~Web searching and information discovery</concept_desc>
       <concept_significance>300</concept_significance>
       </concept>
 </ccs2012>
\end{CCSXML}

\ccsdesc[500]{Information systems~Learning to rank}
\ccsdesc[500]{Information systems~Probabilistic retrieval models}
\ccsdesc[300]{Information systems~Web searching and information discovery}

\keywords{Learning to rank, multinomial logit model, Plackett-Luce model, partial ranking, network formation modeling}

\maketitle

\section{Introduction}

Discrete choice models~\citep{train2009discrete} concern how individuals make choices from a candidate set of alternatives, and have wide applications in many areas, such as recommendation~\citep{mottini2018understanding}, information retrieval~\citep{liu2011learning}, economics~\citep{mcfadden1977modelling}, etc. With the increasing availability of abundant real-world data, more and more often we not only have the record of people's single choice over the candidate set, but also relative rankings of multiple items. Learning discrete choice models from ranking data has thus attracted much research attention~\citep{liu2011learning,hunter2004tutorial,Liu_Zhao_Liao_Lu_Xia_2019}. In this paper, we focus on the problem of learning Multinomial logit (MNL) model, which is one of the most popular discrete choice models and also known as Plackett-Luce model~\citep{luce1959individual,plackett1975analysis}. 

One long-standing technical challenge for learning MNL model from real-world ranking data is that the exact calculation of the MNL likelihood of \emph{partial rankings} is generally intractable when the number of candidate choices is large. This computational challenge has limited the application of MNL model to some special types of ranking data, such as \emph{Top-$K$ ranking}, where the exact order of the top $K$ candidates is required to be known. There are two lines of research that aim to approximate the MNL likelihood of general partial rankings. The first direction is to use \emph{rank breaking}~\citep{soufiani2014computing,khetan2016data,khetan2018generalized}, which usually extracts pairwise comparisons from the general partial rankings and treats the pairs as independent observations. To mitigate the information loss in the extraction of pairwise comparisons, \citet{khetan2018generalized} further proposed \emph{generalized rank breaking}, which extracts maximal ordered partitions from the general partial rankings to preserve more order information. The second direction is to approximate the likelihood by Markov Chain Monte Carlo (MCMC) sampling~\citep{Liu_Zhao_Liao_Lu_Xia_2019}. However, both generalized rank breaking and sampling-based methods suffer from an exponential time complexity in terms of the number of candidate choices, which are not suitable for large-scale social network modeling. Recently, \citet{ma2021learningtorank} demonstrated that the likelihood of \emph{Partitioned-Preference rankings}, which is a special type of partial rankings but is more general than Top-$K$ rankings, can be efficiently approximated through a numerical integral approach. 

In this paper, our first contribution is the development of a scalable method for learning MNL from general partial rankings, by combining the idea of generalized rank breaking~\citep{khetan2018generalized} and the numerial integral approach~\citep{ma2021learningtorank}. We also extend the method to learn the mixture of MNL models by proposing an Expectation-Maximization (EM) algorithm with a novel initialization method to facilitate better convergence. Through simulation study, we demonstrate that the proposed methods are able to achieve significantly lower computational cost while having similar estimation accuracy compared to state-of-the-art sampling-based baseline methods.

We further demonstrate that the proposed methods are particularly useful on applications to choice-based network formation modeling, which is a new and exciting application domain of the MNL model recently introduced by \citet{overgoor2020choosing}. Concretely, we can view the social network formation process as people making choices of their friends. The formation of each new edge in the network can be viewed as a node choosing to connect another node, over all the candidate nodes available. Most prior choice-based network modeling methods assume the temporal orders of all the edge formation are available and interpret the sequentially established edges as top-$K$ rankings~\citep{overgoor2020choosing,overgoor2020scaling}. However, such assumptions may be unrealistic in practice. On one hand, there are a lot of network data where the temporal formation information is partially missing or even totally unavailable. On the other hand, the temporal information does not necessarily imply the level of preference in the friend choices. Our proposed methods, instead, do not rely on such strong assumptions when being used to learn the MNL models. 

We apply the proposed methods to large-scale network formation modeling on both synthetic and real-world networks. We generate synthetic network through mixture of preferential attachment and uniform attachment models. Our empirical results demonstrate that the proposed methods are able to faithfully recover the ground truth parameters of the generative models, while prior choice-based network formation modeling methods fail when the temporal information is missing. Experiments on two real-world networks show that the proposed methods are able to achieve better link predication accuracy than prior baselines, suggesting that the proposed methods are able to achieve a better fitness of the data.

\section{Related Work}

\subsection{Learning MNL from Partial Rankings}
Many algorithms have been proposed to learn MNL model from ranking data~\citep{hunter2004tutorial, czumaj2007faster, azari2013generalized, maystre2015fast,negahban2017rank, khetan2018generalized, ma2021learningtorank} and its mixture~\citep{gormley2008mixture, oh2014learning, zhao2016learning, mollica2017bayesian,Liu_Zhao_Liao_Lu_Xia_2019}. However, few of them is able to learn mixture of MNL models from general partial rankings in a tractable time. Among the existing methods to learn MNL models, some are designed for full rankings~\citep{azari2013generalized, zhao2016learning}; some are designed for special cases of partial rankings~\citep{hunter2004tutorial, maystre2015fast,  mollica2017bayesian, negahban2017rank, khetan2018generalized, ma2021learningtorank}, e.g., Top-$K$ rankings or Partitioned-Preference rankings. But only a few algorithms are proposed to learn MNL models from the most general partial rankings. Moreover, existing methods that can process general partial rankings all suffer from the intractable time complexity. \citet{Liu_Zhao_Liao_Lu_Xia_2019} proposed an algorithm that samples full rankings from partial rankings to learn the MNL model, but this method is not scalable for large dataset due to its exponential time complexity with respect to the number of items. The algorithm by \citet{khetan2018generalized} breaks the general partial rankings into maximal-ordered partitions but has to discard a part of hard-to-evaluate data to confine the time complexity. Though the method proposed by \citet{ma2021learningtorank} is able to learn MNL models in polynomial time, it can handle partitioned preferences at best, which is still not the most general case. Our proposed method combines the algorithms of \citet{khetan2018generalized} and \citet{ma2021learningtorank} to tackle this problem using an efficient numerical integral estimation that can be evaluated in polynomial time.

\subsection{Network Formation Modeling with Discrete Choice Models}

Modeling the formation and growth of network is essential to explore the structure of networks~\citep{holme2019rare, feinberg2020choices, iniguez2019special}. \citet{overgoor2020choosing} recently proposed a framework to model the growth of network as discrete choice of incoming and existing nodes. This framework is general enough to include many existing growth patterns, e.g. preferential attachment. Inspired by this framework, a vast number of researches utilize the temporal information of social networks to better understand the dynamics of the networks~\citep{young2019phase, lee2021dynamics, tomlinson2020learning, reeves2020network, uvzupyte2020test}. Many researchers apply this framework in various domains related to networks such as sociology, physics, medical science and etc~\citep{rosenfeld2020predicting, kawakatsu2021emergence, fu2020modelling, kanakia2020mitigating, overgoor2020structure, lee2021dynamics, holme2021networks, ubaldi2021emergence, mattsson2020financial}. Moreover, many extensions have been made from this framework~\citep{battiston2020networks, gupta2020mixed, tomlinson2020learning, overgoor2020scaling}.

In complement of the existing studies, this work applies our proposed MNL learning methods to handle the situations where temporal information of the edge formation is not fully available or where the temporal information does not directly translate to the preference of nodes. In such situations, learning from partial rankings naturally arises.

\section{Fast Learning of MNL From General Partial Rankings}

In this section, we investigate how to efficiently learn an MNL model from general partial rankings. We develop a scalable method by combining the ideas from two recent studies~\citep{ma2021learningtorank,khetan2018generalized}. We also extend the proposed method to efficiently learn a mixture of MNL model with an EM algorithm.

\subsection{Problem Formulation and Notations}
\label{sec:problem}

We start by introducing the formal definition of the MNL model, and the intractability problem of estimating the MNL likelihood of general partial rankings. 

\vpara{The MNL model.} An MNL model is a popular discrete choice model that characterizes how one makes choices from a group of items. In particular, we assume each individual has an underlying utility score for each of the item, and the observed rankings of the items given by the individuals are noisy version of the utility score order. Now suppose there are $N$ different items and denote the set of items, $\{1,\dots,N\}$, by $\N$. Also denote the set of all possible permutations of $\N$ as $2^{\N}$. The MNL model is defined below.

\begin{definition} [Multinomial logit (MNL) model]
For an individual with $\boldsymbol{w}=[w_1,\dots,w_N]^T$ as the underlying utility scores of the $N$ items, under an \emph{MNL model}, the probability of observing a certain ranking of these items, $(i_1, i_2, \ldots, i_N) \in 2^{\N}$, is defined as
\begin{equation}
p\left(\left(i_1, i_2, \ldots, i_N\right) ; \boldsymbol{w}\right)=\prod_{j=1}^{N} \frac{\exp \left(w_{i_{j}}\right)}{\sum_{l=j}^{N} \exp \left(w_{i_{l}}\right)}.
\label{eq: mnl}
\end{equation}
\end{definition}

\vpara{Partial rankings.} In practice, we often do not observe clear full rankings of the $N$ items. For example, we may know someone's favorite top 5 movies or friends, but rarely their full preferences over all the movies or all the friends. Mathematically, we can represent such partially observed rankings as partially ordered sets (posets). For example, a poset, $\{(3 \succ 2), (3 \succ 5), (4 \succ 1) \}$, which is defined on the set of items $\{1,2,3,4,5\}$ and indicates that the item $3$ is preferred over the item 2 and the item 5; the item 4 is ranked higher than the item 1; but the relative ranking between, e.g., the item 3 and the item 1 is unknown.

\begin{figure}
    \centering
    \includegraphics[width=0.45\textwidth]{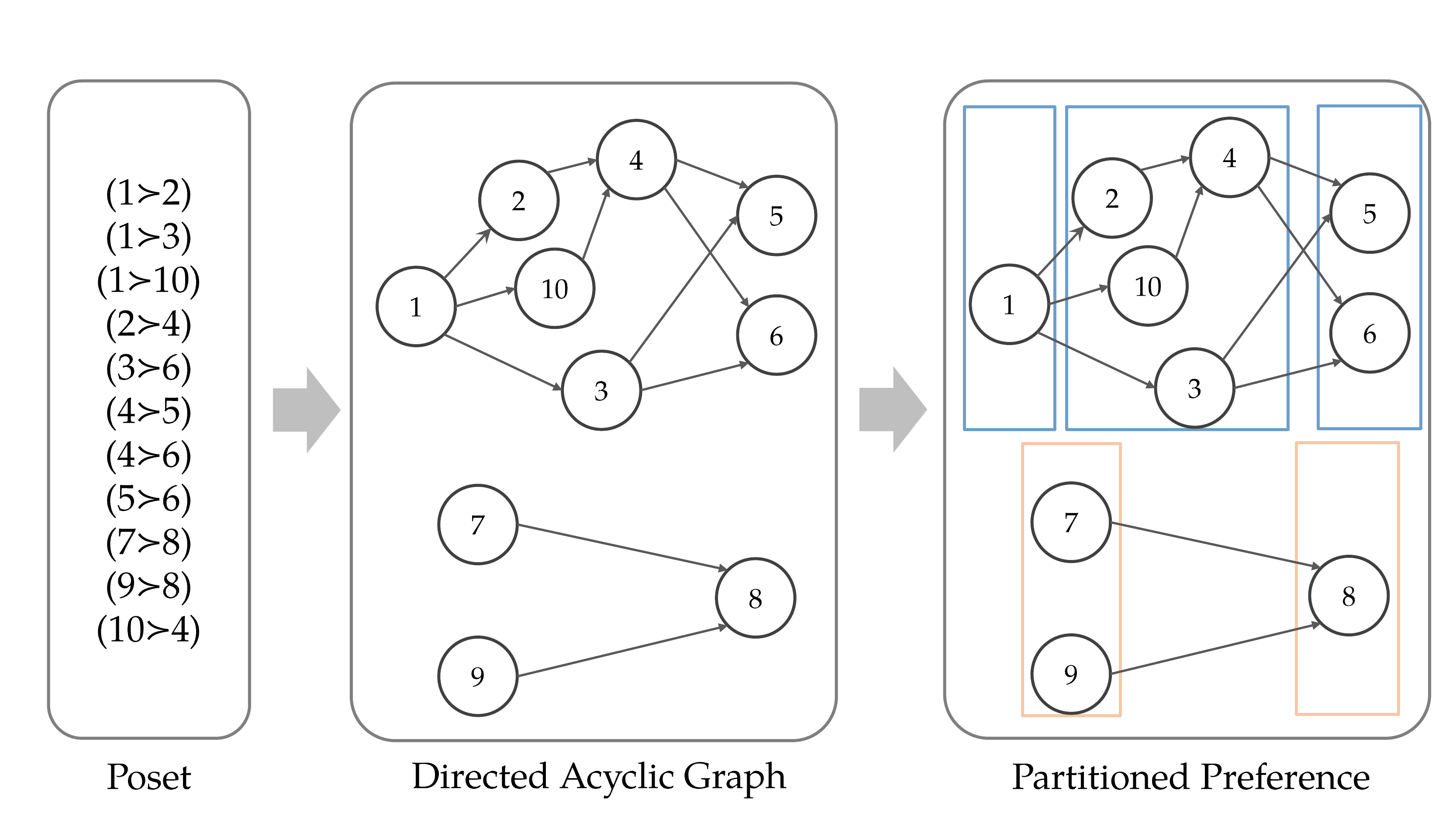}
    \caption{An illustrative example of the conversion from a poset to a DAG, and the extraction of the Partitioned-Preference rankings from a DAG.}
    \label{fig:flowchart}
    \vskip -8pt
\end{figure}

There are two well-known properties of posets, which we will utilize in this paper. First, assuming that there is no tied ranking, each poset corresponds to a \emph{directed acyclic graph} (DAG). See Figure~\ref{fig:flowchart} as an example. Second, we can define \emph{extensions} and \emph{linear extensions} of a poset, where the latter can be used to characterize the full rankings that are consistent with a partial ranking. 

\begin{definition} [Extension of poset]
A poset $\mG_2$ is an \emph{extension} of another poset $\mG_1$ if both posets are defined on the same set of items $X$, and for all elements $x, y \in X$, if $(x \succ y) \in \mG_1$, then $(x \succ y) \in \mG_2$. Moreover, we call $\mG_2$ a \emph{linear extension} of $\mG_1$ if $\mG_2$ is additionally a full order of $X$. And we denote the set of all linear extensions of $\mG_1$ on $X$ as $\Omega(\mG_1; X)$. 
\end{definition}

\vpara{Intractability of the MNL likelihood on partial rankings.} Under an MNL model parameterized by utility scores $\boldsymbol{w}$, the probability of observing a partial ranking $\mG$ is the total probability of full rankings that are consistent with the observed $\mG$, which is given by
\begin{equation}
 P\left(\mG ; \boldsymbol{w}\right) = \sum_{\left(i_{1}, \cdots, i_{N}\right) \in \Omega(\mG;\N)} \prod_{l=1}^{N} \frac{\exp \left(w_{i_{l}}\right)}{\sum_{r=l}^{N} \exp \left(w_{i_{r}}\right)}.
\label{eq: prob}
\end{equation}

For general partial rankings, the number of linear extensions usually grows exponentially with the number of items $N$. Therefore, the summation over $\Omega(\mG;\N)$ in Eq.~(\ref{eq: prob}) makes the likelihood intractable to calculate. Furthermore, except for a few special types of partial rankings, there is no easy way to simplify Eq.~(\ref{eq: prob}) into a tractable closed-form formula. 

\vpara{Partitioned-Preference rankings.} In Section~\ref{sec:polynomial}, we propose a polynomial-time algorithm that is able to efficiently approximate the MNL likelihood of partial rankings for large $N$. The proposed method is built on top of \citet{ma2021learningtorank}, which proposed a numerical approach for fast calculation of the MNL likelihood of a special type of partial rankings, \emph{Partitioned-Preference rankings}~\citep{lebanon2008non,lu2014effective,ma2021learningtorank}, as formally defined below.

\begin{definition} [Partitioned-Preference rankings] 
\label{def:partition}
Given a subset $S\subset\N$, an ordered list of $M$ disjoint partitions of $S$, $(S_1, S_2, \ldots, S_M)$ is called a \emph{Partitioned-Preference ranking} on $S$ if (a) $\cup_{m=1}^M = S$ and $S_m \cap S_{m'} =\emptyset$ for any $m\neq m'$; (b) $S_1\succ S_2 \succ \cdots \succ S_M$, where $S_m\succ S_{m'}$ indicates that any item in the partition $S_m$ has a higher rank than items in the partition $S_{m'}$; (c) the relative ranking of items within the same partition is unknown.
\end{definition}

It is easy to verify that Partitioned-Preference rankings are a strict subset of partial rankings through a counter example below. 
\begin{lemma}
\label{lemma:counter-example}
The partial ranking $\{(3 \succ 2), (3 \succ 5), (4 \succ 1) \}$ is not a Partitioned-Preference ranking on $\{1, 2, 3, 4, 5\}$. 
\end{lemma}
\begin{proof}
    See Appendix~\ref{sec:appendix-counter-example}.
\end{proof}

\subsection{Preliminary: A Numerical Approach for Partitioned-Preference Rankings}
As a preliminary, we briefly review the numerical approach proposed by \citet{ma2021learningtorank}. At the core of this approach is an alternative formulation of the MNL likelihood of a Partitioned-Preference ranking. It is shown that when the partial ranking takes the form of $S_1\succ \cdots S_M$, the likelihood in Eq.~(\ref{eq: prob}) can be rewritten as
\begin{equation}
P\left(S_{1} \succ \cdots \succ S_{M} ; \boldsymbol{w}\right)
=\prod_{m=1}^{M-1} \int_{u=0}^{1} \prod_{i \in S_{m}}\left(1-u^{\exp \left(w_{i}-w_{R_{m+1}}\right)}\right) d u,
\label{eq: int}
\end{equation}
where $R_{m+1}=\cup_{l=m+1}^M S_l$ and $w_{R_{m+1}}=\log\sum_{j\in R_{m+1}}\exp(w_j)$. 

Compared to Eq.~(\ref{eq: prob}), Eq.~(\ref{eq: int}) does not involve the intractable summation over all the linear extensions. While the exact calculation of Eq.~(\ref{eq: int}) is still not tractable due to the integrals involved in the formula, these integrals are one-dimensional and can be efficiently approximated by numerical integration. Furthermore, the gradients of the likelihood with respect to the utility score $\boldsymbol{w}$ can also be approximated by numerical integration. It is also shown that, to achieve any given numerical precision $\varepsilon$ for the likelihood and the gradients, the time complexity of the numerical approach is at most $O(N + \frac{1}{\varepsilon}\left|\cup_{m=1}^{M-1} S_m\right|^3)$, which is much more efficient than the exponential complexity if using Eq.~(\ref{eq: prob}).

\subsection{A Polynomial-Time Algorithm for MNL Likelihood of General Partial Rankings}
\label{sec:polynomial}

As we have seen at the end of Section~\ref{sec:problem}, Partitioned-Preference rankings are a strict subset of partial rankings. Therefore, the numerical approach using Eq.~(\ref{eq: int}) cannot be directly applied to approximate the MNL likelihood of more general partial rankings. However, inspired by \citet{ma2021learningtorank}, we can utilize the idea of \emph{Generalized Rank Breaking} (GRB)~\citep{khetan2018generalized} to extend the numerical approach to derive a polynomial-time algorithm approximating the MNL likelihood of general partial rankings. 

Specifically, \citet{khetan2018generalized} demonstrate that, for each general partial ranking, extracting maximal-ordered partitions (which correspond to Partitioned-Pereference rankings) of the corresponding DAG and calculating likelihood of the extracted Partitioned-Preference rankings can be a good proxy of the original likelihood. Therefore, a natural idea extending the numerical approach~\citep{ma2021learningtorank} to general partial rankings is to first extract (multiple) Partitioned-Preference rankings from the corresponding DAG, and then use the numerical approach to estimate the likelihood of each Partitioned-Preference ranking.

However, a remaining question is how to calculate the joint likelihood of \emph{multiple} Partitioned-Preference rankings. Fortunately, we show in the following Proposition~\ref{prop: disjoint} that the joint likelihood of multiple Partitioned-Preference rankings can be simply decomposed as the product of individual Partitioned-Preference ranking's. 

\begin{prop}
\label{prop: disjoint}
For two sets of Partitioned-Preference rankings $S_1 \succ \cdots \succ S_M$ and $T_1 \succ \cdots \succ T_{M'}$, where $\cup_{m=1}^M S_m = S$ and $\cup_{m=1}^{M'} = T$, and $S, T\subseteq \N$. To ease the notation, we use $\mG_S$ to denote the partial ranking $S_1 \succ \cdots \succ S_M$ and similarly define $\mG_T$. If $S\cap T = \emptyset$, then
\begin{equation}
    \label{eq: disjoint}
    P(\mG_S, \mG_T; \boldsymbol{w}) = P(\mG_S; \boldsymbol{w}) P(\mG_T; \boldsymbol{w}). 
\end{equation}
\end{prop}
\begin{proof}
See Appendix~\ref{sec:appendix-disjoint}.
\end{proof}

We summarize the proposed method for efficiently approximating the MNL likelihood of a general partial ranking $\mG$ in Algorithm~\ref{alg: grb}, which we call it \textbf{Numerical GRB} (NumGRB for short).

\begin{algorithm}
\SetKwInOut{Input}{input}
\SetKwInOut{Output}{Output}
\Input{A partial ranking $\mG$.}
\Output{Approximate log-likelihood $l(\mG)$ for $\mG$.}
\BlankLine
$l(\mG) \gets 0$\;
$\mathcal{C} \gets$ Strongly connected components of $\mG$\;
\For{$\mG' \in \mathcal{C}$}{
    $S\gets [\,]$\;
    \While{$|Vertices(\mG')|>0$}{
    $L\gets$ Lowest common ancestors of all sink nodes of $\mG'$\;
    $S \gets$ $[Vertices(\mG')\backslash L \,,] + S$\;
    $\mG' \gets $ subgraph of $ \mG'$ induced by $L$\;
    }
    $l(\mG) \gets l(\mG) + P(S;\boldsymbol{w})$ where $P(S;\boldsymbol{w})$ is computed by Eq.~(\ref{eq: int});
}

\Return $l(\mG)$
\caption{Numerical GRB (NumGRB)}
\label{alg: grb}
\end{algorithm}

\vpara{Model learning.} For any model parameterizing the utility scores $\boldsymbol{w}$, the model parameters can be learned by minimizing the negative log-likelihood through gradient descent methods. The gradients of the likelihood for each Partitioned-Preference ranking can also be efficiently calculated through numerical integral. In the rest of this paper, we also call this model learning method as NumGRB.

\vpara{Time complexity of Algorithm~\ref{alg: grb}.} The major sub-procedures of Algorithm~\ref{alg: grb} are (1) finding the strongly connected components (SCCs) of the DAG, (2) extracting the maximal-ordered partition for each SCC by recursively finding the common ancestor algorithm, and (3) calculating the log-likelihood (and its gradients) of each Partitioned-Preference ranking through the numerical approach. Specifically, finding the SCCs have time complexity $O(N + E)$, where $E$ is the number of edges of the DAG. Extracting the Partitioned-Preference rankings overall takes a time complexity of $O(N^{3.6})$~\citep{khetan2018generalized} (as common ancestors of all the sink nodes of a DAG can be found in time complexity $O(N^{2.6})$~\citep{czumaj2007faster}). Finally, calculating the log-likelihood (and the gradients) using the numerical approach for each SCC has time complexity at most $O(N^3)$. Therefore, the overall time complexity for Algorithm~\ref{alg: grb} is polynomial-time\footnote{In practice, we find it is usually significantly faster than this worst-case bound.}.

\subsection{Learning Mixture of MNL with EM}
We further extend the proposed method to learn mixture of MNL models, which often leads to better fitness of data~\citep{zhao2018learning} and is of particular interest for the application to network formation modeling~\citep{overgoor2020choosing,overgoor2020scaling}. Formally, a mixture of MNL model is defined below.

\begin{definition}[Mixture of MNL]
Given an integer $k\geq 1$, a mixture of MNL model contains two parts of parameters: $\boldsymbol{\pi}=(\pi_1,\dots, \pi_k)$, where $\pi_r\geq 0$ for $1\leq r\leq k$, and $\sum_{1\leq r\leq k}\pi_r=1$; $(\boldsymbol{w}^{(1)},\dots,\boldsymbol{w}^{(k)})$, where $\boldsymbol{w}^{(r)} \in \mathbb{R}^N$ is the parameter of the $r$-th MNL component. The probability of observing a partial ranking $\mG$ is defined as
\begin{equation}
    P\left(\mG ; \boldsymbol{\pi},(\boldsymbol{w}^{(1)},\dots,\boldsymbol{w}^{(k)})\right) = \sum_{1\leq r\leq k}\pi_r P(\mG;\boldsymbol{w}^{(r)}).
\end{equation}
\end{definition}

We propose to learn the mixture of MNL with an expectation-maximization (EM) algorithm, where the likelihood of each MNL component is estimated by Algorithm~\ref{alg: grb}. In practice, however, standard EM algorithm with random initialization often converges to bad local optima, which is also observed in the learning of other types of mixture models~\citep{jin2016local}. As a remedy, we further propose a clustering-based initialization method that greatly stabilizes the convergence of EM. 

\vpara{Clustering-based initialization.} It is common to initialize an EM algorithm with clustering methods such as K-means~\citep{melnykov2012initializing}. In our case, however, applying K-means requires us to measure the distance between partial rankings, which is rarely studied. Following the idea that the same item is likely to have similar relative ranks in two different rankings if the two ranking are generated from the same MNL component, we propose a novel ranking distance to assist the clustering of ranking data. Formally, we define the relative rank of an item in a Partitioned-Preference ranking below.

\begin{definition}[Relative rank]
Consider a Partitioned-Preference ranking $\mG_S=(S_1\succ\cdots\succ S_M)$, we define the relative rank $r_{\mG_S}(i)$ as
\begin{equation}\label{eq: rr}
    r_{\mG_S}(i) = \dfrac{m(i)-1}{M-1}
\end{equation}
where $1\le m(i)\le M$ and the $m(i)$-th partition $S_{m(i)}$ contains item $i$.
\end{definition}
Intuitively, $r_{\mG_S}(i)$ measures the relative rank of an item by the rank of its partition. Then we can define the ranking distance.

\begin{definition}[Ranking distance]
Consider two Partitioned-Preference rankings $\mG_S=(S_1\succ\cdots\succ S_M)$ with $S=\cup_{m=1}^M S_m$, and $\mG_T=(T_1\succ\dots\succ T_{M'})$ with $T=\cup_{m'=1}^{M'} T_{m'}$, we define the \emph{ranking distance} $d$ between $\mG_S$ and $\mG_T$ as
\begin{equation} \label{eq: ranking-dist}
d(\mG_S,\mG_T) \nonumber = \sqrt{\dfrac{\sum_{i\in S \cap T}\left(r_{\mG_S}(i)-r_{\mG_T}(i)\right)^2}{|S\cap T|}}. 
\end{equation}
\end{definition}
With the ranking distance defined in Eq.~(\ref{eq: ranking-dist}), we can now use K-means to cluster all the observed partial rankings based on the extracted Partitioned-Preference rankings. To initialize the EM algorithm, we learn a single MNL model on each cluster and then initialize an MNL component with the learned parameters. The full EM algorithm (named \textbf{NumGRB-EM}) is summarized in Algorithm~\ref{alg: em}.

\begin{algorithm}
\SetKwInOut{Input}{input}
\SetKwInOut{Output}{Output}
\Input{A list of $n$ observed partial rankings $G=(\mG_1,\dots, \mG_n)$, number of components $k$, number of iterations $B$.}
\Output{Parameters of the mixture of MNL: $\boldsymbol{\pi},(\boldsymbol{w}^{(1)},\dots,\boldsymbol{w}^{(k)})$.}
\BlankLine
Apply clustering-based initialization to $\boldsymbol{w}^{(1)},\dots,\boldsymbol{w}^{(k)}$\;
Initialize $\pi_1=\dots=\pi_k=1/k$\;
\For{$b=1,\dots,B$}{
    $\forall 1\leq j\leq n, 1\leq r\leq k$, compute $P(\mG_j;\boldsymbol{w}^{(r)})$ by Algorithm~\ref{alg: grb}\;
    % E-step
    $\forall 1\leq j\leq n, 1\leq r\leq k$, compute $\gamma_{jr}$ by
    \[\gamma_{jr}\gets \dfrac{P(\mG_j;\boldsymbol{w}^{(r)})\pi_r}{\sum_{s=1}^{k}P(\mG_j;\boldsymbol{w}^{(s)})\pi_s}\;\]
    $\forall 1\leq r\leq k$, update $\pi_r$ by
    \[\pi_r\gets \dfrac{\sum_j \gamma_{jr}}{n}\;\]
    % M-step
    Update $\boldsymbol{w}^{(r)}$ for each $r$, weighing each data point $\mG_j$ with its class responsibility $\gamma_{jr}$, i.e.,
    \[\boldsymbol{w}^{(r)} \gets \underset{\boldsymbol{w}}{\arg\max} \sum_{j=1}^{n}\gamma_{jr}\log P(\mG_j;\boldsymbol{w})\;\]
}

\Return $\boldsymbol{\pi}, \boldsymbol{w}^{(1)}, \dots, \boldsymbol{w}^{(k)}$.
\caption{NumGRB-EM}
\label{alg: em}
\end{algorithm}

\section{Simulation Study}
\label{sec:simulation}
In this section, we empirically verify the estimation accuracy and computation cost of the proposed methods through simulation. 

\subsection{Experiment Setup}

We compare the proposed methods, \textbf{NumGRB} for single MNL and \textbf{NumGRB-EM} for mixture of MNL, with a state-of-the-art baseline \textbf{ELSR-Gibbs}~\citep{Liu_Zhao_Liao_Lu_Xia_2019}, which is a sampling-based method that can be applied to learn both single and mixture of MNL. 

We closely follow the simulation setup by \citet{Liu_Zhao_Liao_Lu_Xia_2019} for our experiments. We conduct experiments on synthetic data generated by both single and mixture of MNL. 

We treat the the utility scores $\boldsymbol{w}$ of the (mixture of) MNL as free parameters, and respectively apply the proposed methods and ELSR-Gibbs to learn $\boldsymbol{w}$. We implement the proposed methods with PyTorch~\citep{paszke2019pytorch} and use AdaGrad optimizer~\citep{duchi2011adaptive} with an initial learning rate of $0.5$. For the baseline ELSR-Gibbs, we use the official implementation\footnote{\url{https://github.com/zhaozb08/MixPL-SPO}} released by the authors of \citet{Liu_Zhao_Liao_Lu_Xia_2019}. We run our experiments on a GeForce RTX 2080 Ti GPU and an Intel(R) Xeon(R) Gold 6230 CPU @ 2.10GHz CPU.

We evaluate the estimation accuracy of each method by the mean squared error (MSE) between the learned parameters and ground truth parameters, both normalized with a softmax function. For the mixture of MNL, the average MSE on all components is reported. We also report the running time of each method to evaluate the computation cost.

\subsection{Experiments on Single MNL}
\label{sec: single}

\vpara{Synthetic data generated by single MNL.} Given the number of items $N$, we first generate the ground truth utility scores $\boldsymbol{w}=(w_1, \dots, w_N)$ uniformly on $[-2,2]$. Then we draw $n$ samples of full rankings from an MNL model parameterized by $\boldsymbol{w}$ following Eq.~(\ref{eq: mnl}). To obtain partial rankings, we sample from all the $\frac{N(N-1)}{2}$ pairwise comparisons determined by the full rankings, and keep each pairwise comparison with probability $p$ independently. 

\begin{figure}[h]
        \vskip -5pt
        \begin{subfigure}[b]{0.23\textwidth}
                \includegraphics[width=\linewidth]{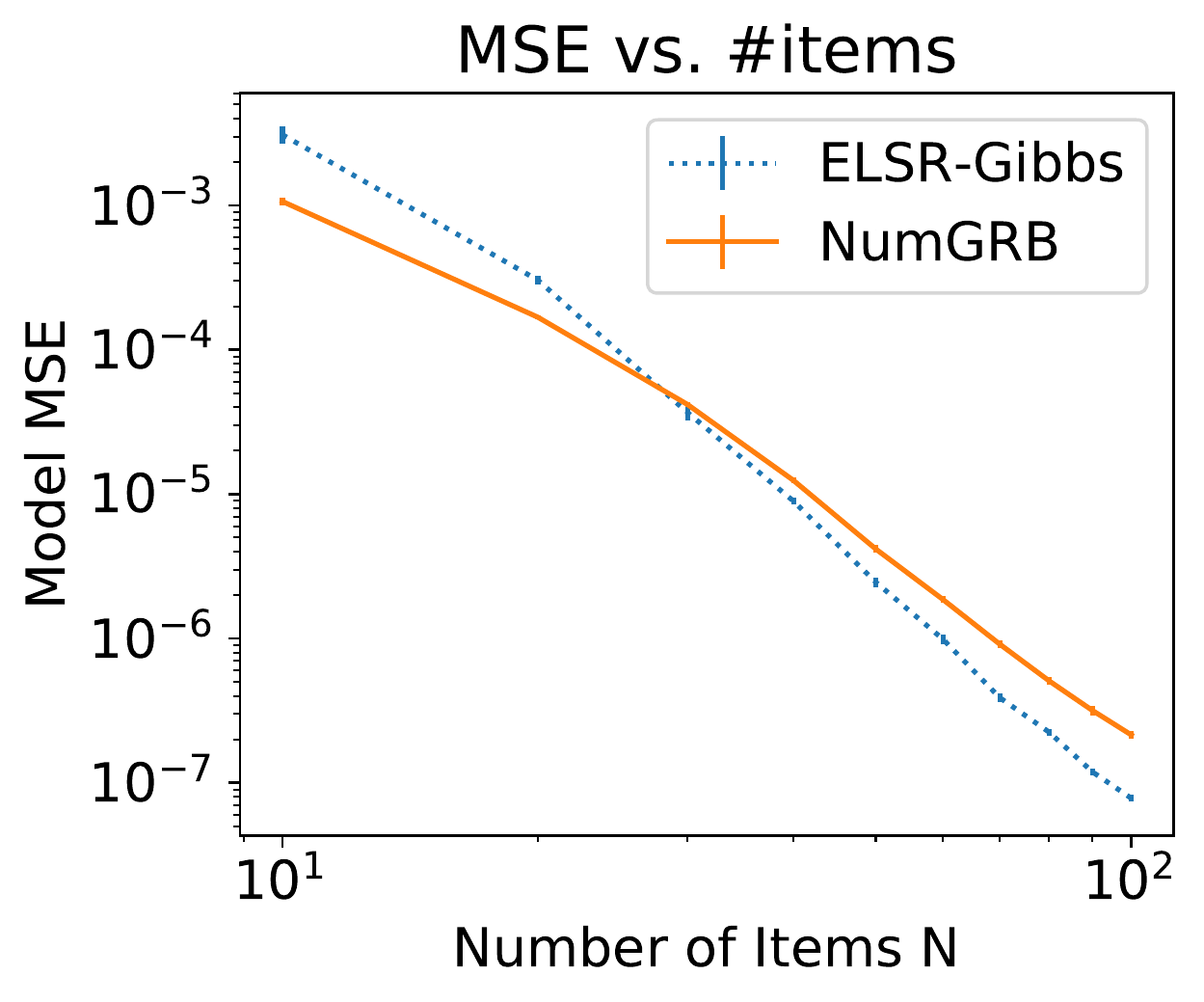}
        \end{subfigure}%
        \begin{subfigure}[b]{0.22\textwidth}
                \includegraphics[width=\linewidth]{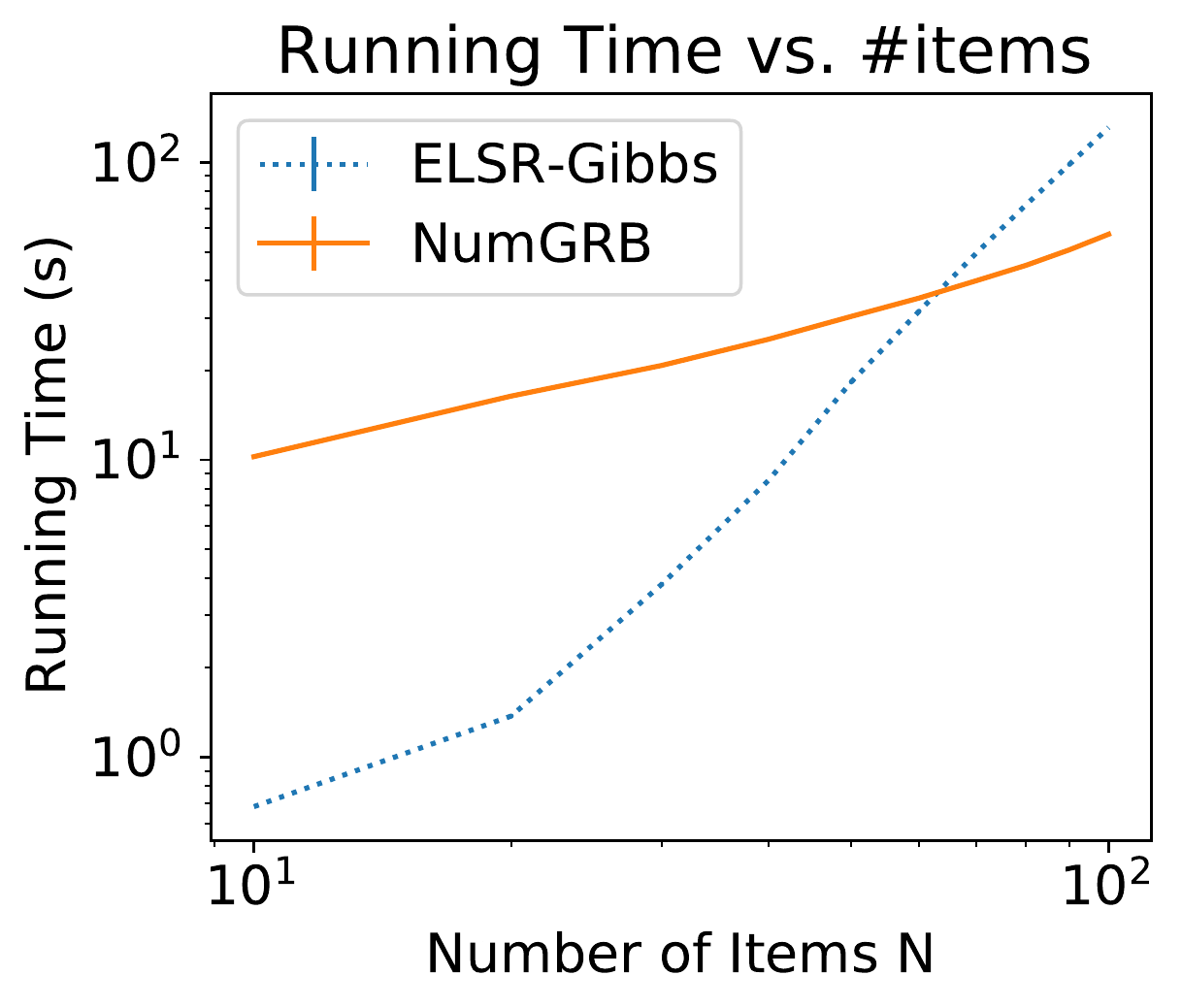}
        \end{subfigure}%
        \caption{MSE and running time of NumGRB and ELSR-Gibbs with varying number of items $N$. The sample size $n$ is fixed as 5,000. The sample rate $p$ is fixed as 0.25. Both axes of the plots are in the logarithmic scale with base 10. The results are averaged over 50 different random seeds and error bars (barely visible) indicate the standard error of the mean.}
        \label{fig:single1}
        \vskip -8pt
\end{figure}

\vpara{Results with varying number of items $N$.} Figure~\ref{fig:single1} shows the MSE and running time of the proposed NumGRB and the baseline ELSR-Gibbs with $N$ varying from 10 to 100. 

In terms of MSE, the proposed NumGRB achieves similar estimation accuracy on most configurations of $N$. The trend that MSE is decreasing with $N$ is due to the artifact that the softmax scores of parameters sum to 1 and larger $N$ implies smaller normalized scores. 

In terms of running time, while the proposed NumGRB has a heavier overhead when $N$ is small, the computational cost of ELSR-Gibbs quickly increases and surpasses NumGRB when $N$ becomes larger. It is worth noting that the figure is a log-log plot. The linear curve of NumGRB indicates that it has a polynomial-time complexity, while the curve of ELSR-Gibbs apparently grows faster than linear, as ELSR-Gibbs does not have a polynomial-time guarantee. We are not able to provide experiment comparison for $N$ significantly larger than $100$ because the baseline ELSR-Gibbs suffers from numerical errors for large $N$. In contrast, the proposed NumGRB is still numerically stable for at least $N=1000$.

\begin{figure}[h]
        \begin{subfigure}[b]{0.21\textwidth}
                \includegraphics[width=\linewidth]{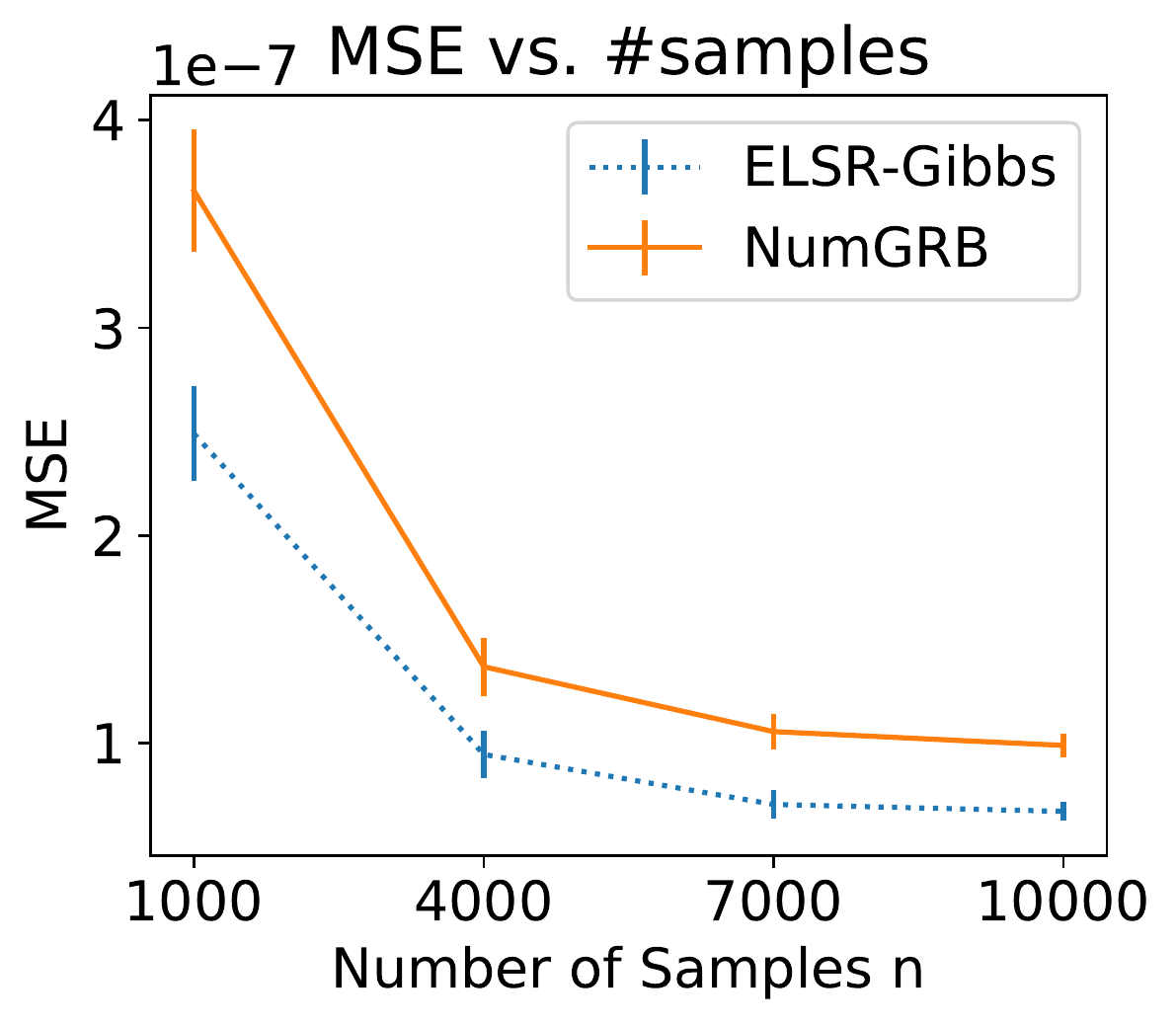}
        \end{subfigure}%
        \begin{subfigure}[b]{0.22\textwidth}
                \includegraphics[width=\linewidth]{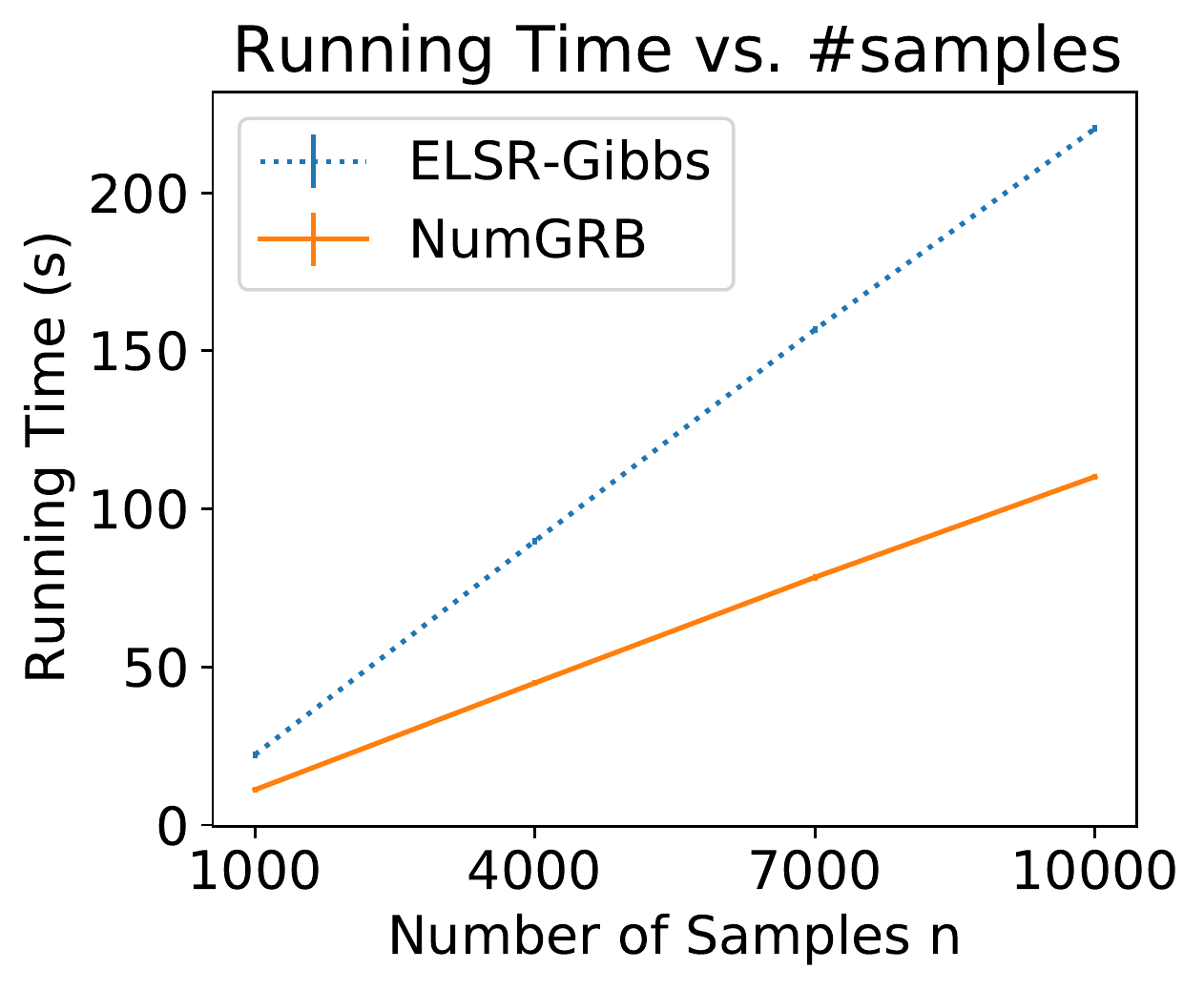}
        \end{subfigure}%
        \caption{MSE and running time of NumGRB and ELSR-Gibbs with varying number of samples $n$. The number of items $N$ is fixed as 60. The sample rate $p$ is fixed as 0.5. Note both axes of the plots are in linear scale. The results are averaged over 10 different random seeds and error bars indicate the standard error of the mean.}
        \label{fig:single2}
        \vskip -10pt
\end{figure}

\vpara{Results with varying number of samples $n$.} Figure~\ref{fig:single2} shows the MSE and running time of the proposed NumGRB and the baseline ELSR-Gibbs with $n$ varying from 1,000 to 10,000. In terms of MSE, while ELSR-Gibbs appears to have better sample efficiency, the difference is not very large (note the y-axis is now in linear scale to show the difference). In terms of running time, both methods appear to increase linearly with $n$, but the computation cost of ELSR-Gibbs increases faster than the proposed NumGRB.

\subsection{Experiments on Mixture of MNL}

\vpara{Synthetic data generated by mixture of MNL.} Following \citet{Liu_Zhao_Liao_Lu_Xia_2019}, we generate synthetic data from a 3-mixture of MNL model. We assign equal weights to the 3 mixture components, i.e., $\pi_r=1/3$ for $r=1,2,3$. The configuration for each MNL component is the same as the single MNL described in Section~\ref{sec: single}. When drawing each of the $n$ samples of full rankings, we first randomly select a mixture component according to $\pi$, then draw a full ranking from this MNL component. After we get the $n$ full rankings, we sample partial rankings similarly as for the single MNL.

\begin{figure}
    \begin{subfigure}[b]{0.22\textwidth}
            \includegraphics[width=\linewidth]{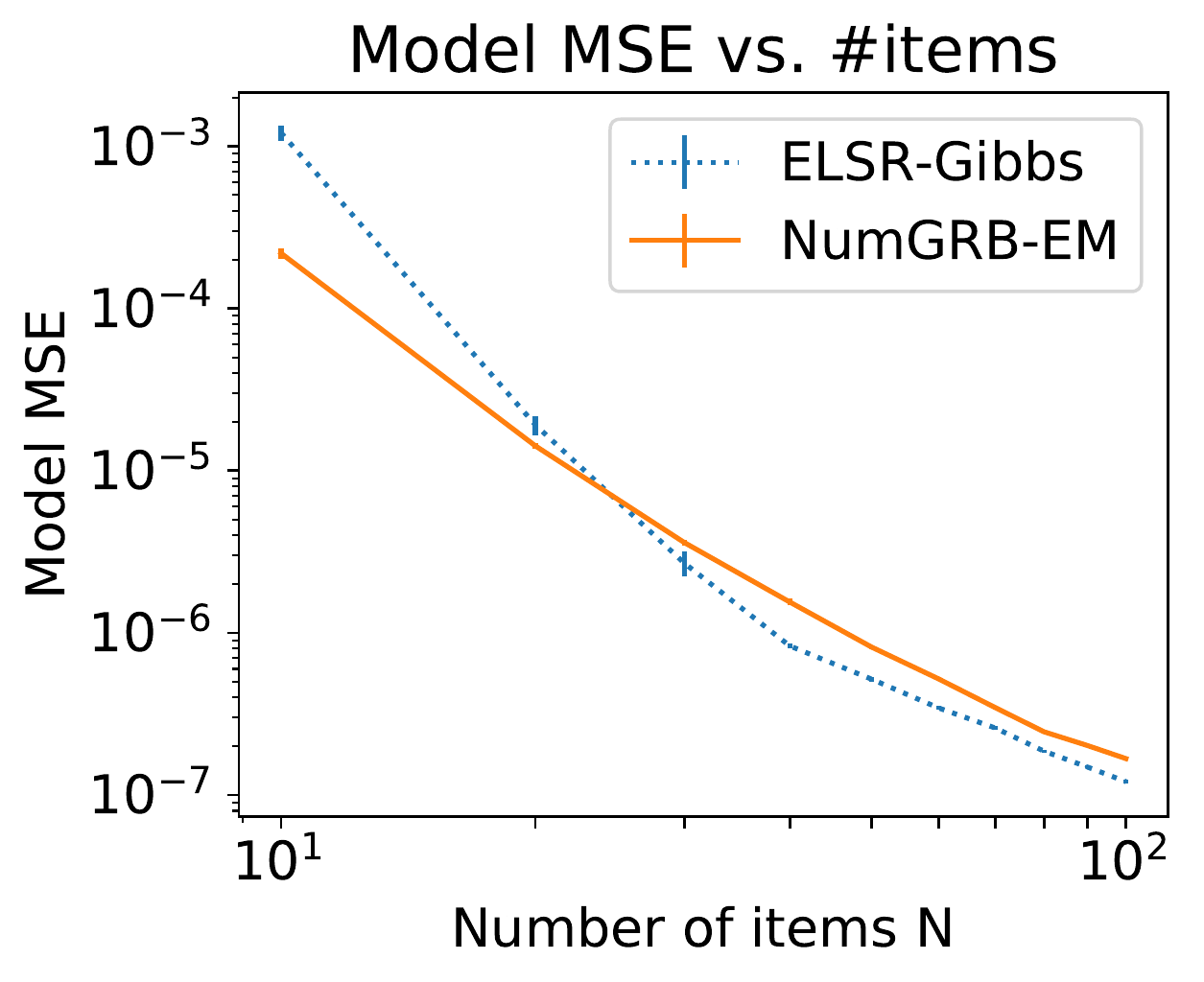}
    \end{subfigure}%
    \begin{subfigure}[b]{0.22\textwidth}
            \includegraphics[width=\linewidth]{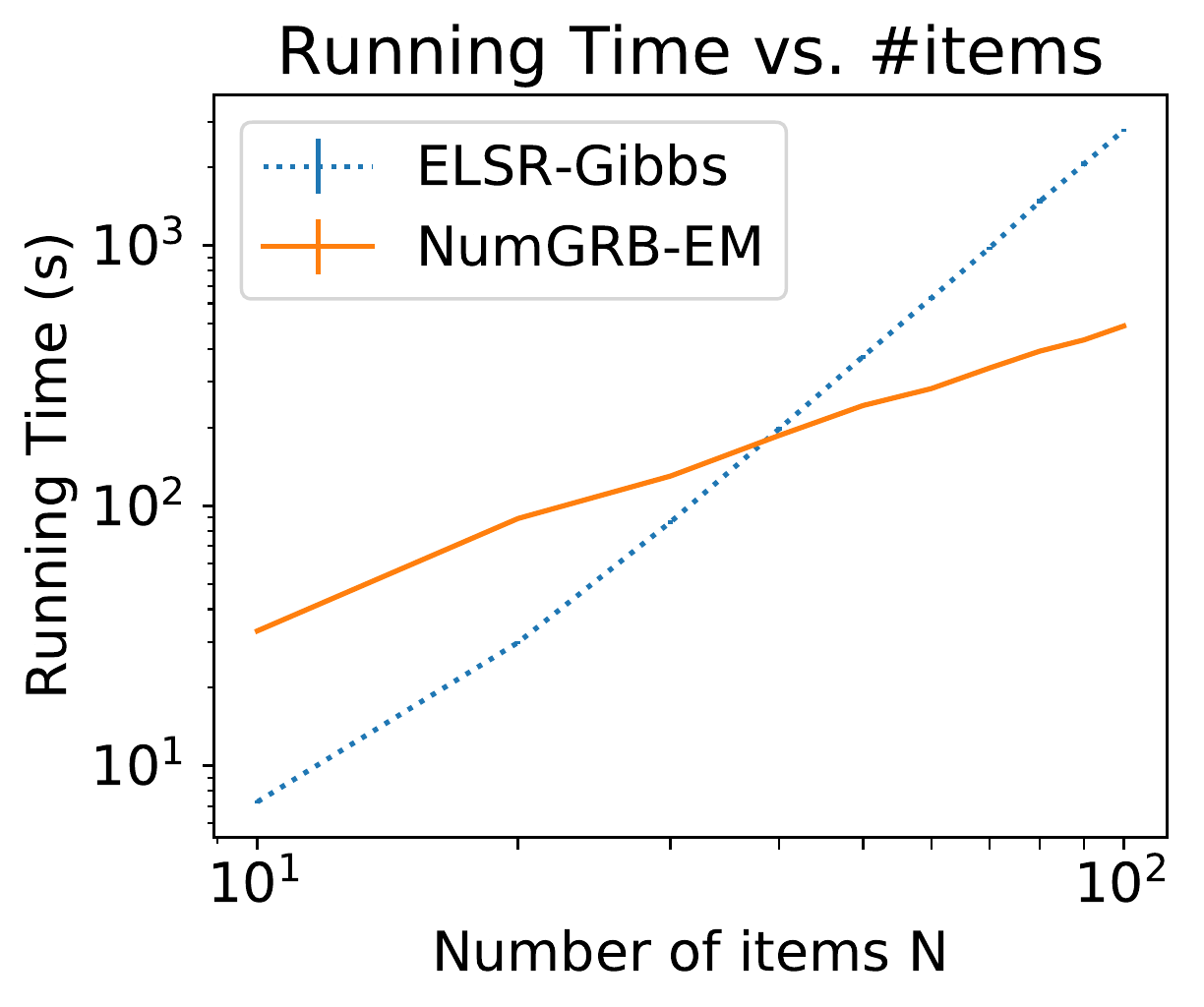}
    \end{subfigure}%
    \caption{MSE (averaged over three MNL components) and running time of NumGRB-EM and ELSR-Gibbs with varying number of items $N$. The sample size $n$ is fixed as 5,000. The sample rate $p$ is fixed as 0.5. Both axes are in the logarithmic scale with base 10. The results are averaged over 50 different random seeds with the the worst 10 results discarded.}
    \label{fig:mix}
    \vskip -10pt
\end{figure}

\vpara{Results.} Figure~\ref{fig:mix} shows the MSE and running time of the proposed NumGRB-EM and the baseline ELSR-Gibbs with $N$ varying from 10 to 100. We first note that, as the negative log-likelihood of the mixture model is non-convex, both the proposed NumGRB-EM and the baseline ELSR-Gibbs occasionally converge to bad local optima where some of the mixture components are completely missed by the model. In such a case, the performance is of orders of magnitude worse compared to the performance when all 3 MNL components are recovered. Such bad local optima dominate the average MSE over different random trials, making the comparison of average MSE meaningless. For a fair and meaningful comparison, we discard the worst 10 trials among the 50 random trials for each method with each $N$, in order to make sure trials that do not recover all three components are removed. We also summarize the number of trials that converge to bad local optima for each method. For the baseline ELSR-Gibbs, 82 out of the 500 training (50 random trials on 10 different setting of $N$) fail to recover all three components. For the proposed NumGRB-EM, only 46 out of the 500 training fail to recover all three components.

In terms of the average MSE, we again observe that the proposed NumGRB-EM achieves similar estimation accuracy as the baseline ELSR-Gibbs on most configurations of $N$. In terms of the running time, the proposed NumGRB-EM obtains even larger advantage over ELSR-Gibbs, compared to the experiments on single MNL.

\section{Application to Network Formation Modeling}
We further apply the proposed methods to network formation modeling on both synthetic and real network data. 

\subsection{Network Formation as Discrete Choice}

We first review the basic discrete-choice-based network formation modeling framework proposed by \citet{overgoor2020choosing}. The central idea is to view the formation of a directed edge from node $i$ to node $j$ as the process that $i$ chooses $j$ over a set of alternative nodes, which can then be modeled by a single MNL or a mixture of MNL. One merit of this framework is that it elegantly unifies many existing network formation models, such as preferential attachment~\citep{barabasi1999emergence}, uniform attachment~\citep{callaway2001randomly}, and latent space model~\citep{hoff2002latent}. More importantly, this framework allows one to combine node features with different network formation mechanisms into a single model and estimate their relative importance.

Next, we review two concrete examples, preferential attachment model and uniform attachment model. 

\begin{definition}[Preferential attachment]
If the formation of new edges in a directed graph follows \emph{preferential attachment}, the new arriving node connects to an existing node $j$ with a probability proportional to a power of their degree $d_j$, i.e.,
\begin{equation}
    P(j; V)=\dfrac{d_j^\alpha}{\sum_{l\in V}d_l^\alpha} = \dfrac{\exp(\alpha\log d_j)}{\sum_{l\in V}\exp(\alpha \log d_l)},
\end{equation}
where $V$ denotes the candidate set. 
\end{definition}

\begin{definition}[Uniform Attachment]
If the formation of new edges in a directed graph follows uniform attachment, the new arriving node connects to any existing node $j$ in the candidate set $V$ with the same probability, i.e.,
\begin{equation}
    P(j; V)=\dfrac{1}{|V|} = \dfrac{\exp(1)}{\sum_{l\in V}\exp(1)}.
\end{equation}
\end{definition}

As can be seen in their definitions, the preferential attachment and uniform attachment models can be viewed as MNL models with utility score for $j$ being $\alpha \log d_j$ and $1$ respectively. 

More generally, under the preferential attachment model, the probability for a node sequentially choosing $(j_1, j_2, \ldots, j_K)$ to connect can be written as
\[P((j_1, j_2, \ldots, j_K); V) = \prod_{k=1}^K \frac{\exp(\alpha \log d_{j_k})}{\sum_{l\in V\setminus \{j_1, \ldots, j_{k-1}\}} \exp(\alpha \log d_l)}.\]

However, one limitation of most existing methods under this framework~\citep{overgoor2020choosing,overgoor2020scaling} is that learning the model from an observed network requires the knowledge of full rankings for everyone's all existing edges, which is unrealistic in many real-world scenarios. The proposed methods in this work, NumGRB and NumGRB-EM, can be applied to learn from more general partial rankings of existing edges in a network.

\begin{table}[]
	\resizebox{\linewidth}{!}{%
\begin{tabular}{c|c|ccc|ccc}
	\hline
	& \multirow{2}{*}{Para. (True Val.)}& \multicolumn{3}{c}{NumGRB-EM} & \multicolumn{3}{c}{Naive} \\ \cline{3-8}
	                                                  &                                & 4-Mix   & 2-Mix   & Single  & 4-Mix   & 2-Mix   & Single  \\ \hline
	\multirow{10}{*}{\begin{tabular}[c]{@{}c@{}}r=0.2  \\ 
	p=0.2\end{tabular}}                               & \multirow{2}{*}{UA (0.04)}     & 0.055   &         &         & 0.127   &         &         \\
	                                                  &                                & (0.004) &         &         & (0.018) &         &         \\ \cline{2-8}
	                                                  & \multirow{2}{*}{UA-FoF (0.16)} & 0.179   &         &         & 0.317   &         &         \\
	                                                  &                                & (0.024) &         &         & (0.018) &         &         \\ \cline{2-8}
	                                                  & \multirow{2}{*}{PA (0.16)}     & 0.174   & 0.239   & 1.000   & 0.091   & 0.213   & 1.000   \\
	                                                  &                                & (0.007) & (0.009) & (0.000) & (0.013) & (0.002) & (0.000) \\ \cline{2-8}
	                                                  & \multirow{2}{*}{PA-FoF (0.64)} & 0.593   & 0.761   &         & 0.466   & 0.787   &         \\
	                                                  &                                & (0.025) & (0.009) &         & (0.023) & (0.002) &         \\ \cline{2-8}
	                                                  & \multirow{2}{*}{$\alpha$ (1)}  & 0.989   & 0.877   & 0.814   & 1.518   & 0.945   & 0.829   \\
	                                                  &                                & (0.019) & (0.008) & (0.006) & (0.098) & (0.021) & (0.006) \\ \hline
	\multirow{10}{*}{\begin{tabular}[c]{@{}c@{}}r=0.5 \\
	p=0.5\end{tabular}}                               & \multirow{2}{*}{UA (0.25)}     & 0.262   &         &         & 0.256   &         &         \\
	                                                  &                                & (0.008) &         &         & (0.013) &         &         \\ \cline{2-8}
	                                                  & \multirow{2}{*}{UA-FoF (0.25)} & 0.249   &         &         & 0.252   &         &         \\
	                                                  &                                & (0.013) &         &         & (0.010) &         &         \\ \cline{2-8}
	                                                  & \multirow{2}{*}{PA (0.25)}     & 0.249   & 0.502   & 1.000   & 0.247   & 0.516   & 1.000   \\
	                                                  &                                & (0.008) & (0.004) & (0.000) & (0.015) & (0.014) & (0.000) \\ \cline{2-8}
	                                                  & \multirow{2}{*}{PA-FoF (0.25)} & 0.240   & 0.498   &         & 0.246   & 0.484   &         \\
	                                                  &                                & (0.017) & (0.004) &         & (0.005) & (0.014) &         \\ \cline{2-8}
	                                                  & \multirow{2}{*}{$\alpha$ (1)}  & 0.949   & 0.557   & 0.632   & 1.199   & 0.589   & 0.644   \\
	                                                  &                                & (0.028) & (0.033) & (0.011) & (0.042) & (0.021) & (0.009) \\ \hline
	\multirow{10}{*}{\begin{tabular}[c]{@{}c@{}}r=0.8 \\
	p=0.8\end{tabular}}                               & \multirow{2}{*}{UA (0.64)}     & 0.521   &         &         & 0.476   &         &         \\
	                                                  &                                & (0.052) &         &         & (0.037) &         &         \\ \cline{2-8}
	                                                  & \multirow{2}{*}{UA-FoF (0.16)} & 0.128   &         &         & 0.106   &         &         \\
	                                                  &                                & (0.002) &         &         & (0.009) &         &         \\ \cline{2-8}
	                                                  & \multirow{2}{*}{PA (0.16)}     & 0.308   & 0.817   & 1.000   & 0.344   & 0.792   & 1.000   \\
	                                                  &                                & (0.052) & (0.005) & (0.000) & (0.039) & (0.005) & (0.000) \\ \cline{2-8}
	                                                  & \multirow{2}{*}{PA-FoF (0.04)} & 0.043   & 0.183   &         & 0.074   & 0.208   &         \\
	                                                  &                                & (0.008) & (0.005) &         & (0.006) & (0.005) &         \\ \cline{2-8}
	                                                  & \multirow{2}{*}{$\alpha$ (1)}  & 0.742   & 0.234   & 0.353   & 0.773   & 0.275   & 0.323   \\
	                                                  &                                & (0.124) & (0.023) & (0.006) & (0.099) & (0.014) & (0.005) \\ \hline
	\multirow{10}{*}{\begin{tabular}[c]{@{}c@{}}r=1   \\
	p=0\end{tabular}}                                 & \multirow{2}{*}{UA (0)}        & 0.020   &         &         & 0.050   &         &         \\
	                                                  &                                & (0.010) &         &         & (0.009) &         &         \\ \cline{2-8}
	                                                  & \multirow{2}{*}{UA-FoF (0)}    & 0.000   &         &         & 0.000   &         &         \\
	                                                  &                                & (0.000) &         &         & (0.000) &         &         \\ \cline{2-8}
	                                                  & \multirow{2}{*}{PA (1)}        & 0.980   & 1.000   & 1.000   & 0.950   & 1.000   & 1.000   \\
	                                                  &                                & (0.010) & (0.000) & (0.000) & (0.009) & (0.000) & (0.000) \\ \cline{2-8}
	                                                  & \multirow{2}{*}{PA-FoF (0)}    & 0.000   & 0.000   &         & 0.000   & 0.000   &         \\
	                                                  &                                & (0.000) & (0.000) &         & (0.000) & (0.000) &         \\ \cline{2-8}
	                                                  & \multirow{2}{*}{$\alpha$ (1)}  & 0.994   & 0.993   & 0.998   & 1.045   & 0.996   & 0.989   \\
	                                                  &                                & (0.016) & (0.014) & (0.012) & (0.008) & (0.013) & (0.006) \\ \hline
\end{tabular}
}
\caption{Estimated model parameters by NumGRB-EM and Naive on the synthetic network data. The results are averaged over 5 random trials and the numbers in brackets indicate the standard error of the mean. The synthetic network data are always generated by 4-mixtures of UA, UA-FoF, PA, and PA-FoF with ground truth $\alpha=1$. We generate data with 4 settings for different values of $r$ and $p$. The columns with heads ``4-Mix'', ``2-Mix'', and ``Single'' refer to the learned models being specified with 4-mixtures, 2-mixtures (PA and PA-FoF), and a single model (PA), respectively. The estimated parameters include the mixture weights for the four components, as well as the parameter $\alpha$ shared by PA and PA-FoF. There are blank entries for ``2-Mix'' and ``Single'' as they are not specified.}
\label{tab: syn-net}
\vskip -25pt
\end{table}

\subsection{Experiments on Synthetic Network Data}
\label{sec:synthetic-network}

We first apply the proposed NumGRB-EM on synthetic network data generated by mixture of 4 variants of network formation models. Through our experiments, we demonstrate that (1) the proposed principled method for learning from partial rankings outperforms naive approximation methods; (2) correctly specifying the mixture components not only affects the learning of mixture weights but also improves the model parameter estimation of each mixture component.

\vpara{Synthetic network data generation.} We simulate the growth of a directed network with a synthetic $(r, p)$-model following~\citet{overgoor2020choosing}. When a new edge is formed, with probability $p$, it is formed by uniform attachment, and with probability $1-p$, it is formed by preferential attachment with $\alpha=1$. After choosing the attachment pattern, we choose the candidate set $V$ to fully determine the mixture component: with probability $r$, the candidate set is all nodes in the network that have not been connected by the source node, while with probability $1-r$, the candidate set is restricted to the friends-of-friends (FoF) of the source node. Therefore, the synthetic data is generated by a mixture of 4 edge-choice distributions, which we denote as UA (uniform attachment, mixture weight $pr$), PA (preferential attachment, mixture weight $(1-p) r$), UA-FoF (uniform attachment restricted to FoF, mixture weight $p(1-r)$), PA-FoF (preferential attachment restricted to FoF, mixture weight $(1-p)(1-r)$). 

We generate synthetic network data with varying values of $(r,p)$ pairs. Due to the space limit, the detailed data generation procedure is provided in Appendix~\ref{sec:appendix-syn-network-generation}.

\vpara{Experiment setup.} The learning task of this experiment is to learn 5 parameters of the mixture model: the mixture weights for the 4 components as well as a shared model parameter $\alpha$ for PA and PA-FoF. And we apply the proposed NumGRB-EM method to learn the parameters from the data. 

For the baseline, we note that ELSR-Gibbs is not scalable enough for this experiment setup. We instead implement a naive baseline method (named as \textbf{Naive}) that treats each new edge formation as an independent top-one ranking, such that the calculation of likelihood becomes feasible (but with information lost). For both NumGRB-EM and Naive, we further apply them to learn a 2-mixture (PA and PA-FoF) and a single model (PA only), in order to investigate their performance under mis-specified settings.

\vpara{Results.} The experiment results are summarized in Table~\ref{tab: syn-net}. Comparing NumGRB-EM and Naive, the proposed NumGRB significantly outperforms Naive in terms of both the estimation of mixture weights and the distribution parameter $\alpha$ in most cases. This verifies that a principled method for calculating the likelihood of partial rankings is critical for the model parameter estimation accuracy. 

Regarding the specification of the mixture models, we start with the setting $r=1$ and $p=0$, which is a special case where the ground truth model is actually a single model following PA. It is not surprising that all models fit well in this case because they all contain the PA component. It is worth noting that applying NumGRB-EM to learn 4-mixtures or 2-mixtures is able to faithfully recover the single model case. In more general settings where $0 < r, p < 1$, we can see that, the 4-mixture models trained by the proposed NumGRB-EM method outperform their mis-specified counterparts not only in terms of the mixture weights, but also in terms of the distribution parameter $\alpha$.

Together, the experiments suggest that both the proposed partial ranking likelihood estimation and the extension to mixture models are helpful for learning choice-based network formation models.

\subsection{Experiments on Real-World Network Data}
\label{sec:real-network}
We further apply the proposed method on two real-world network datasets, Flickr~\citep{mislove2008growth} and Microsoft Academic Graph\footnote{The AMiner project~\citep{tang2008arnetminer,sinha2015overview}: \url{https://aminer.org/open- academic- graph}.}, and compare with the method by~\citet{overgoor2020choosing}. For both datasets, we closely follow most of the data processing steps by \citet{overgoor2020choosing}. The key difference between our method for fitting the data and \citet{overgoor2020choosing}'s lies in how we interpret the ranking of target nodes given the observed network.

For Flickr, \citet{overgoor2020choosing} use the temporal order of the edge formation in the network to construct the full ranking of the target nodes connected by the source node, with the underlying assumption that earlier formed edges correspond to more preferred target nodes. We think, however, this assumption might be too strong, as it might be hard to tell the relative ranking of two target nodes with edges formed within a small time window. We instead use the temporal window of the edge formation to construct Partitioned-Preference rankings. 

For Microsoft Academic Graph, there is even no temporal order available for the edge formation, as the citation links from a source node (a new paper) are almost always formed simultaneously (at its publication), and therefore the ranking of target nodes (exiting papers) is naturally a Partitioned-Preference ranking with 2 partitions (cited or not cited). \citet{overgoor2020choosing} fit the data with the naive method as we described in Section~\ref{sec:synthetic-network}, which treats each edge as an independent top-one ranking. 

For the experiments, we parameterize the utility scores with linear models on node features used by \citet{overgoor2020choosing}. We report both the learned linear coefficients and the precision@k metrics for link prediction. Due to space limit, more detailed experiment setups are provided in Appendix~\ref{sec:appendix-exp-setup}.

\begin{table}[]
	\centering
	\resizebox{\linewidth}{!}{%
		\begin{tabular}{c|cccc|cccc}
			\hline
			& \multicolumn{4}{c|}{\citet{overgoor2020choosing}}&   \multicolumn{4}{c}{NumGRB}                                          \\ \cline{2-9} 
			              & \#1      & \#2                   & \#3      & \#4      & \#1*                  & \#2*                 & \#3*                 & \#4*                          \\ \hline
			log Followers & $1.149$  &                       & $0.7150$ & $0.536$  & $0.806$               &                      & $0.4810$             & $0.3470$                      \\
			Has Degree    & $-0.580$ &                       & $-0.631$ & $-1.745$ & $-3.556$              &                      & $-3.527$             & $-3.335$                      \\
			Reciprocal    & $8.419$  & $8.347$               & $8.197$  & $7.903$  & $5.854$               & $4.491$              & $5.614$              & $5.171$                       \\
			Is FoF        &          & $6.120$               & $3.955$  &          &                       & $3.918$              & $2.888$              &                               \\
			2 Hops        &          &                       &          & $6.290$  &                       &                      &                      & $3.778$                       \\
			3 Hops        &          &                       &          & $2.851$  &                       &                      &                      & $1.348$                       \\
			4 Hops        &          &                       &          & $0.583$  &                       &                      &                      & $-0.877$                       \\
			5 Hops        &          &                       &          & $-0.585$ &                       &                      &                      & $-1.92$                      \\
			$\geq$ 6 Hops &          &                       &          & $-1.122$ &                       &                      &                      & $-2.264$                      \\ \hline
			Precision@1   & $0.7606$ & $\underline{0.7898}$  & $0.8238$ & $0.8262$ & $\underline{0.7622}$  & $\underline{0.7898}$ & $\underline{0.8414}$ & $\underline{\mathbf{0.8448}}$ \\
			Precision@3   & $0.6215$ & $0.6385$              & $0.6737$ & $0.6756$ & $\underline{0.6227}$  & $\underline{0.6398}$ & $\underline{0.6796}$ & $\underline{\mathbf{0.6846}}$ \\
			Precision@5   & $0.5444$ & $0.5573$              & $0.5880$ & $0.5895$ & $ \underline{0.5464}$ & $\underline{0.5592}$ & $\underline{0.5926}$ & $\underline{\mathbf{0.5948}}$ \\ \hline
		\end{tabular}
	}
	\caption{Estimated model parameters and precision@k for link prediction on Flickr. Model \#1-4 are trained by the baseline method by \citet{overgoor2020choosing}. Model \#1-4* are trained by the proposed NumGRB. For the precision@k metrics, the bold markers denote the best performance and the underline markers denote the better performance between the corresponding baseline and proposed methods.}
	\label{tab:flickr}
	\vskip -10pt
\end{table}

\begin{table}[]
\resizebox{\linewidth}{!}{%
\begin{tabular}{c|cccc|cccc}
\hline
                         & \multicolumn{4}{c|}{\citet{overgoor2020choosing}}                                                                                    & \multicolumn{4}{c}{NumGRB}                                                                               \\ \cline{2-9} 
                         & \#1                          & \#2                          & \#3                          & \#4                           & \#1*                         & \#2*                         & \#3*                         & \#4*                         \\ \hline
log Citations            & $0.717$                      & $0.794$                      & $1.052$                      & $1.044$                       & $0.649$                      & $0.611$                      & $0.932$                      & $0.922$                      \\
Has Degree               & $1.684$                      & $1.677$                      & $1.862$                      & $1.830$                       & $0.280$                      & $0.250$                      & $1.495$                      & $1.462$                      \\
Has Same Author          &                              & $6.523$                      & $5.928$                      & $5.913$                       &                              & $4.870$                      & $4.430$                      & $4.424$                      \\
log Age                  &                              &                              & $-1.096$                     & $-1.069$                      &                              &                              & $-1.458$                      & $-1.443$                     \\
log Max Papers &                              &                              &                              & \multirow{2}{*}{$0.029$}                       &                              &                              &                              & \multirow{2}{*}{$0.030$}                 \\
 by Authors &                              &                              &                              &                    &                              &                              &                              &                     \\ \hline
Precision@1              & $\underline{0.2495}$                     & $0.4745$                     & $\underline{0.4960}$                     & $0.4970$                      & $0.2490$                     & $\underline{0.4835}$                     & $\underline{0.4960}$                     & $\underline{\textbf{0.5020}}$                     \\
Precision@3              & $0.1955$                     & $0.4128$                     & $0.4213$                     & $0.4235$                      & \underline{$0.1957$}                     & $\underline{0.4132}$                     & $\underline{0.4293}$                     & $\underline{\textbf{0.4295}}$                     \\
Precision@5              & $0.1628$ & $\underline{0.3558}$ & $0.3654$ & $0.3664$ & $\underline{0.1629}$ & $0.3545$ & $\underline{\textbf{0.3735}}$ & $\underline{0.3731}$ \\
Precision@10             & $0.1238$                     & $\underline{0.2609}$                     & $0.2683$                     & $0.2690$                      & $\underline{0.1238}$                     & $0.2587$                     & $\underline{0.2756}$                     & $\underline{\textbf{0.2758}}$                     \\ \hline
\end{tabular}
}
\caption{Estimated model parameters and precision@k for link prediction on Microsoft Academic Graph. The markers are the same as in Table~\ref{tab:flickr}.}
\label{tab:mag}
\vskip -20pt
\end{table}

\vpara{Results.} The experiment results on Flickr and Microsoft Academic Graph are provided in Table~\ref{tab:flickr} and Table~\ref{tab:mag} respectively. 

In terms of the precision@k metrics for link prediction, we observe that treating the rankings of target nodes as Partitioned-Preference rankings and train the models with the proposed NumGRB outperforms the method by \citet{overgoor2020choosing} when there are more features included in the model. The difference is smaller when there are very few features, which is not surprising as the model is not able to learn sophisticated patterns in these cases. 

In terms of the estimated linear coefficients, qualitatively the explanations by models trained with both methods have similar trend. Quantitatively, however, some of the estimated coefficients by the two methods differ by a large amount. For example, comparing Model \#3 and \#3* in Table~\ref{tab:flickr}, the relative strengths of the coefficients between ``Has Degree'' and ``Reciprocal'' are very different. Researchers facing the results given by Model \#3 may interpret that the impact of having zero followers is marginal. However, the results by Model \#3* tell a largely different story. And the latter is probably a better fit of the data (as suggested by the link prediction accuracy).

\section{Conclusion}

In this paper, we propose novel methods for fast learning of the MNL model and its mixture from general partial rankings. We verify both the effectiveness and efficiency of the proposed methods through simulation studies. We also demonstrate that the proposed methods can be very useful on applications to choice-based network formation modeling, where the formation of edges is viewed as discrete choices from (mixture of) MNL models. Through experiments on both synthetic networks and real-world networks, we show that efficient learning (mixture of) MNL models from partial rankings is necessary for proper estimation of the model parameters, which are critical for downstream network analysis tasks in scientific research. 

\subsection*{Acknowledgements}
This work was in part supported by the National Science Foundation under grant number 1633370.

\bibliographystyle{ACM-Reference-Format}
\bibliography{reference}

\clearpage
\appendix
\section{Appendix}
\subsection{Proof of Lemma~\ref{lemma:counter-example}}
\label{sec:appendix-counter-example}
\begin{proof} [Proof of Lemma~\ref{lemma:counter-example}]
For the sake of contradiction, assume $\{(3 \succ 2), (3 \succ 5), (4 \succ 1) \}$ is a Partitioned-Preference ranking on $\{1, 2, 3, 4, 5\}$. 
    
As we do not know the relative ranking between 3 and 4, so according to (a) and (b) in Definition~\ref{def:partition}, 3 and 4 must be in the same partition. Similarly, 3 and 1 must be in the same partition. However, we also know $4 \succ 1$, which contradicts with (c) in Definition~\ref{def:partition}. 
\end{proof}
\subsection{Proof of Proposition~\ref{prop: disjoint}}
\label{sec:appendix-disjoint}

We first introduce a well-known connection between the MNL model and Gumbel distribution~\citep{yellott1977relationship}. For an MNL model parameterized by $\boldsymbol{w}$, let $g_{\boldsymbol{w}_1}, g_{\boldsymbol{w}_2}, \ldots, g_{\boldsymbol{w}_N}$ denote $N$ independent Gumbel random variables following 
\[\text{Gumbel}(\boldsymbol{w}_1), \text{Gumbel}(\boldsymbol{w}_2), \ldots, \text{Gumbel}(\boldsymbol{w}_N),\]
where $\text{Gumbel}(w)$ refers to the Gumbel distribution with location parameter $w$. Then the probability of observing a full ranking $(i_1, \ldots, i_N)\in 2^{\N}$ under the MNL model is equal to the probability that $g_{\boldsymbol{w}_{i_1}} > g_{\boldsymbol{w}_{i_2}} > \cdots > g_{\boldsymbol{w}_{i_N}}$, i.e., 
\[P((i_1, \ldots, i_N);\boldsymbol{w}) = P(g_{\boldsymbol{w}_{i_1}} > g_{\boldsymbol{w}_{i_2}} > \cdots > g_{\boldsymbol{w}_{i_N}}).\]
So there is a bijective mapping between the full rankings and the orders of the Gumbel variables.

\begin{proof} [Proof of Proposition~\ref{prop: disjoint}]

Given a set of Gumbel variables, 
\[g_{\boldsymbol{w}_1}, g_{\boldsymbol{w}_2}, \ldots, g_{\boldsymbol{w}_N},\]
for any $A\subseteq \N$, define 
\[\bar{m}(A) \triangleq \max_{i\in A} g_{w_i}, \quad \ubar{m}(A) \triangleq \min_{i\in A} g_{w_i}.\]

Further denote $\tilde{S}_m = \cup_{l=m}^M S_l$ and $\tilde{T}_m = \cup_{l=m}^{M'} T_l$.

Based on the bijective mapping between the full rankings and the orderes of Gumbel variables, it is easy to verify that the event of $\Omega(\mG_S; \N)$ corresponds to the event of Gumbel variables satisfying
\[\ubar{m}(S_1) > \bar{m}(\tilde{S}_2), \ubar{m}(S_2) > \bar{m}(\tilde{S}_3), \ldots, \ubar{m}(S_{M-1}) > \bar{m}(\tilde{S}_M).\]
While the event of $\Omega(\mG_T; \N)$ corresponds to the event of Gumbel variables satisfying
\[\ubar{m}(T_1) > \bar{m}(\tilde{T}_2), \ubar{m}(T_2) > \bar{m}(\tilde{T}_3), \ldots, \ubar{m}(T_{M'-1}) > \bar{m}(\tilde{T}_{M'}).\]

Then we have
\begin{align*}
    & P(\mG_S, \mG_T; \boldsymbol{w}) \\
    =& P(\ubar{m}(S_1) > \bar{m}(\tilde{S}_2), \ldots, \ubar{m}(S_{M-1}) > \bar{m}(\tilde{S}_M),\\
    & \phantomrel{==}    \ubar{m}(T_1) > \bar{m}(\tilde{T}_2), \ldots, \ubar{m}(T_{M'-1}) > \bar{m}(\tilde{T}_{M'})).
\end{align*}
Further, since $S\cap T = \emptyset$, and the Gumbel variables are mutually independent, we have
\begin{align*}
    & P(\ubar{m}(S_1) > \bar{m}(\tilde{S}_2), \ldots, \ubar{m}(S_{M-1}) > \bar{m}(\tilde{S}_M), \\
    &\phantomrel{==}\ubar{m}(T_1) > \bar{m}(\tilde{T}_2), \ldots, \ubar{m}(T_{M'-1}) > \bar{m}(\tilde{T}_{M'})) \\
    =& P\left(\ubar{m}(S_1) > \bar{m}(\tilde{S}_2), \ldots, \ubar{m}(S_{M-1}) > \bar{m}(\tilde{S}_M)\right)\\
    &\phantomrel{==} \cdot P\left(\ubar{m}(T_1) > \bar{m}(\tilde{T}_2), \ldots, \ubar{m}(T_{M'-1}) > \bar{m}(\tilde{T}_{M'})\right) \\
    =&  P(\mG_S; \boldsymbol{w}) P(\mG_T; \boldsymbol{w}). 
\end{align*}

\end{proof}

\subsection{Detailed Synthetic Network Generation Procedure}
\label{sec:appendix-syn-network-generation}
Given a pair of $(r, p)$, we generate the network with the following procedure. We first generate an initial network using Erd\H{o}s-R{\'e}nyi random graph~\citep{erdos1960evolution} with $1000$ nodes and probability $0.005$, and we also randomly choose 20 nodes and increase their edges by 50 to 80 to increase the variety of node degrees. We then randomly select half of the nodes in the network as source nodes to form new edges. For each source node, we sample 1 out of the 4 mixture components, UA, PA, UA-FoF, and PA-FoF as its edge-choice distribution, and make each node form 5 new edges\footnote{A few nodes may form less than 5 edges when their candidate set is FoF and the size of the candidate set is less than 5.}. Throughout the new edge formation process, we use the node degrees in the initial network when calculating the choice distribution. We treat the new edge formation for each source node as a Partitioned-Preference ranking with 2 partitions: the target nodes pointed by the new edges are in the first partition while the nodes that are not connected by the source node are in the second partition. We vary the values of $r$ and $p$ to get multiple settings of synthetic data. For each setting, we repeat the data generation with 5 random seeds.

\subsection{Detailed Experiment Setup for Experiments on Real-World Networks}
\label{sec:appendix-exp-setup}
\vpara{Experiment setup for Flickr.} We follow the data processing steps of \citet{overgoor2020choosing} to extract the network data. In this network, we consider 4 types features: the log of number of the followers, whether the following edge is reciprocal, whether the follower and the followed user is friends-of-friends, and the path length (hop) from the new follower to the followed user before the new edge is made.  It is possible that a user, has never been followed before. Thus, we use the censored log function, i.e., define $\log(0)=0$, and add another feature ``has degree'' to distinguish the zero case.  We sample $50000$ users in a period (20 days) that have following events, and record all the new edges formed by them. Edges formed in the same 10 days are considered in one partition. Edges formed earlier are assumed to be more preferable than later ones. We then sample $100$ edges uniformly at random from the dataset as the negative samples for each user. The testing set is sampled in the same way, with the sampling period later for 20 days. We train 4 linear models shown in Table.~\ref{tab:flickr} with the proposed NumGRB method and evaluate them using precision@[$1,3,5$] on testing set. 

\vpara{Experiment setup for Microsoft Academic Graph.}
Microsoft Academic Graph contains the events of citations between publications. Our task is to predict the references of a new publication among a large candidate set. We continue to follow the steps of \citet{overgoor2020choosing} to extract the citation data in the domain of climatology. We include 4 features of the candidate references: the log of number of citations at the time of citation, whether the paper shares authors with the candidate reference, the log age of the paper in years at the time of citation, and the maximum number of publications ever by any one of the authors at the time of publication. We still use the censored log function to avoid log of zero. Additionally, we introduce another feature, whether the candidate reference has degree, to distinguish if the number of citation is 0 or 1. We sample 12,000 papers after 2010 as source nodes and record references for each paper. We sample 5000 negative samples for each sampled paper uniformly at random. The data is then split into training set and testing set by time. We train 4 linear models shown in Table.~\ref{tab:mag} using the proposed NumGRB method and report precision@[$1,3,5,10$] on testing set.

\end{document}